\documentclass{article}

\usepackage[preprint]{neurips_2023}

\usepackage[utf8]{inputenc} 
\usepackage[T1]{fontenc}    
\usepackage{hyperref}       
\usepackage{url}            
\usepackage{booktabs}       
\usepackage{amsfonts}       
\usepackage{nicefrac}       
\usepackage{microtype}      
\usepackage{xcolor}         

\usepackage{amsmath}

\usepackage{algorithm}
\usepackage{algorithmic}
\usepackage{svg}
\usepackage{appendix}
\usepackage{bm}

\newtheorem{theorem}{Theorem}
\newtheorem{lemma}{Lemma}

\newtheorem{proof}{Proof}

\title{Distance-rank Aware Sequential Reward Learning for Inverse Reinforcement Learning with Sub-optimal Demonstrations}

\author{%
Lu Li$^{1}$\thanks{Equal Contribution} \quad Yuxin Pan$^{2\ast}$ \quad Ruobing Chen$^{3}$ \quad Jie Liu$^{4}$ \quad Zilin Wang$^{1}$ \quad Yu Liu$^{3}$\thanks{Corresponding authors} \quad Zhiheng Li$^{1\dagger}$ \\
$^{1}$Tsinghua University, Beijing, China \\
$^{2}$Hong Kong University of Science and Technology, Hong Kong SAR, China \\
$^{3}$SenseTime Research, Beijing, China \\
$^{4}$Chinese University of Hong Kong, Hong Kong SAR, China\\
}

\begin{document}

\maketitle

\begin{abstract}
  Inverse reinforcement learning (IRL) aims to explicitly infer an underlying reward function based on collected expert demonstrations. Considering that obtaining expert demonstrations can be costly, the focus of current IRL techniques is on learning a better-than-demonstrator policy using a reward function derived from sub-optimal demonstrations. However, existing IRL algorithms primarily tackle the challenge of trajectory ranking ambiguity when learning the reward function. They overlook the crucial role of considering the degree of difference between trajectories in terms of their returns, which is essential for further removing reward ambiguity. Additionally, it is important to note that the reward of a single transition is heavily influenced by the context information within the trajectory. To address these issues, we introduce the Distance-rank Aware Sequential Reward Learning (DRASRL) framework. Unlike existing approaches, DRASRL takes into account both the ranking of trajectories and the degrees of dissimilarity between them to collaboratively eliminate reward ambiguity when learning a sequence of contextually informed reward signals. Specifically, we leverage the distance between policies, from which the trajectories are generated, as a measure to quantify the degree of differences between traces. This distance-aware information is then used to infer embeddings in the representation space for reward learning, employing the contrastive learning technique. Meanwhile, we integrate the pairwise ranking loss function to incorporate ranking information into the latent features. Moreover, we resort to the Transformer architecture to capture the contextual dependencies within the trajectories in the latent space, leading to more accurate reward estimation. Through extensive experimentation on MuJoCo tasks and Atari games, our DRASRL framework demonstrates significant performance improvements over previous state-of-the-art IRL methods. Notably, we achieve an impressive 0.99 correlation between the learned reward and the ground-truth reward. The resulting policies exhibit a remarkable $95\%$ performance improvement in MuJoCo tasks.
\end{abstract}

\section{Introduction}
Reinforcement learning (RL) is a learning paradigm that involves automatically acquiring an optimal policy through interactions with an environment, guided by a predefined reward function. In scenarios where manually designing a reward function is challenging, inverse reinforcement learning (IRL) offers a viable alternative by inferring a reward function from expert demonstrations. The inferred reward function captures the implicit preferences and goals of the expert, enabling the agent to generalize their behavior and learn a policy that aligns with their expertise. However, collecting a large amount of expert demonstration data can be challenging and expensive, especially in domains such as robotic operations~\cite{kober2013reinforcement} and autonomous driving~\cite{kiran2021deep, pan2020navigation}. Consequently, current IRL methods~\cite{brown2019extrapolating,brown2019drex, chen2021learning} prioritize Learning from Sub-optimal Demonstrations (LfSD). These methods make use of sub-optimal demonstrations to derive a reward function that partially captures the underlying task goals. The objective is to learn a policy that extrapolates beyond the limitations of the provided sub-optimal demonstrations by utilizing the derived reward function. In contrast, Imitation learning (IL) approaches~\cite{hussein2017imitation} are inherently limited by the available demonstrations, which poses a challenge when attempting to incorporate LfSD techniques into IL frameworks.

Existing methods for inferring a reward function from sub-optimal demonstrations typically rely on utilizing the rank information of trajectory pairs. These methods employ pairwise ranking loss functions to effectively encapsulate the underlying preferences over each trajectory for the viable reward learning. The rank information used in these methods can be annotated by human experts~\cite{brown2019extrapolating}. However, manually estimating the rank for each pair of trajectories requires a substantial amount of human intervention, which hinders the scalability and applicability of these methods in real-world settings. D-REX~\cite{brown2019drex} extends this approach by automating the creation of ranked demonstrations through the incorporation of various levels of noise into a behavioral cloning (BC) policy. However, these methods primarily focus on reducing reward ambiguity in reward learning by considering trajectory ranking alone. As a result, they neglect the potential of utilizing the degree of difference between trajectories to eliminate the reward ambiguity in sub-optimal demonstrations. In addition, previous IRL methods place their primary emphasis on estimating the reward for each transition individually, and fail to take into account the contextual information present within the same trace during the reward learning process.

In the standard RL paradigm, the reward function can evaluate each pair of actions given the same state by providing two scalar rewards. By comparison, trajectory ranking information only partially fulfills the functionality of the standard reward function by indicating which trajectory has a higher cumulative reward. On the other hand, the standard reward function can effectively indicate the difference in cumulative rewards between trajectories, providing a more comprehensive measure of trajectory preference. However, the degree of dissimilarity in terms of returns between trajectories, which may offer valuable insights into the relative quality of traces, cannot be fully captured by solely considering trajectory ranking. \emph{The ranking approach enables us to discern the preferred trajectory, while considering the degree of difference provides valuable insights into the magnitude of preference that one trajectory holds over the other.} We argue that both trajectory ranks and the degree of differences are essential for obtaining a comprehensive understanding of the quality and performance of trajectories. In addition, in certain scenarios such as partial observability, the evaluation of each transition is significantly influenced by the contextual information present within the same trajectory. We contend that it is essential to sequentially model the reward signals to incorporate context dependencies into reward learning.

In this paper, we present the Distance-rank Aware Sequential Reward Learning (DRASRL) framework as a solution to address these challenges. Our proposed framework takes into consideration both trajectory ranking and the degree of dissimilarity between trajectory pairs to effectively mitigate ambiguity during the learning process of a sequence of contextually informed reward signals. By incorporating both aspects collaboratively, we aim to achieve a more robust and accurate reward learning approach that captures the nuances and preferences present in the demonstrations. To be specific, the absence of the difference in corresponding returns contributes to the persistent challenge of ambiguity arising from the degree of dissimilarity between each pair of trajectories. We thus make a reasonable assumption that the degree of difference in terms of returns between a pair of trajectories is proportional to the distance of the policies from which those trajectories are generated. To effectively leverage the insights provided by the assumption, we utilize the policy distance as a means to quantify the distance between trajectories in the latent space. Instead of using policy distance as the regression target for predicting return differences, this method utilizes contrastive learning techniques to learn a distance-aware representation for reward learning by using the policy distance as a "soft label". In addition to encoding the absolute differences in latent space, DRASRL integrates a pairwise ranking loss to capture relative rank-aware representations. Moreover, to incorporate the contextual relationship between transitions into reward learning, the Transformer architecture~\cite{vaswani2017attention} is leveraged by taking a sequence of state-action pairs as inputs and generating sequential reward signals as outputs. We conduct extensive experiments across standard benchmarks, including Atari games~\cite{mnih2013playing} and MuJoCo locomotion tasks~\cite{todorov2012mujoco}. Compared to previous state-of-the-art techniques, our approach exhibits superior reward accuracy and improved performance of the trained policies. Notably, DRASRL yields a remarkable 0.99 correlation with the ground-truth reward, and the policies showcase a substantial $95\%$ performance improvement in the MuJoCo HalfCheetah task.

\begin{figure*}[t]
\begin{center}
\centerline{\includegraphics[width=0.9\linewidth]{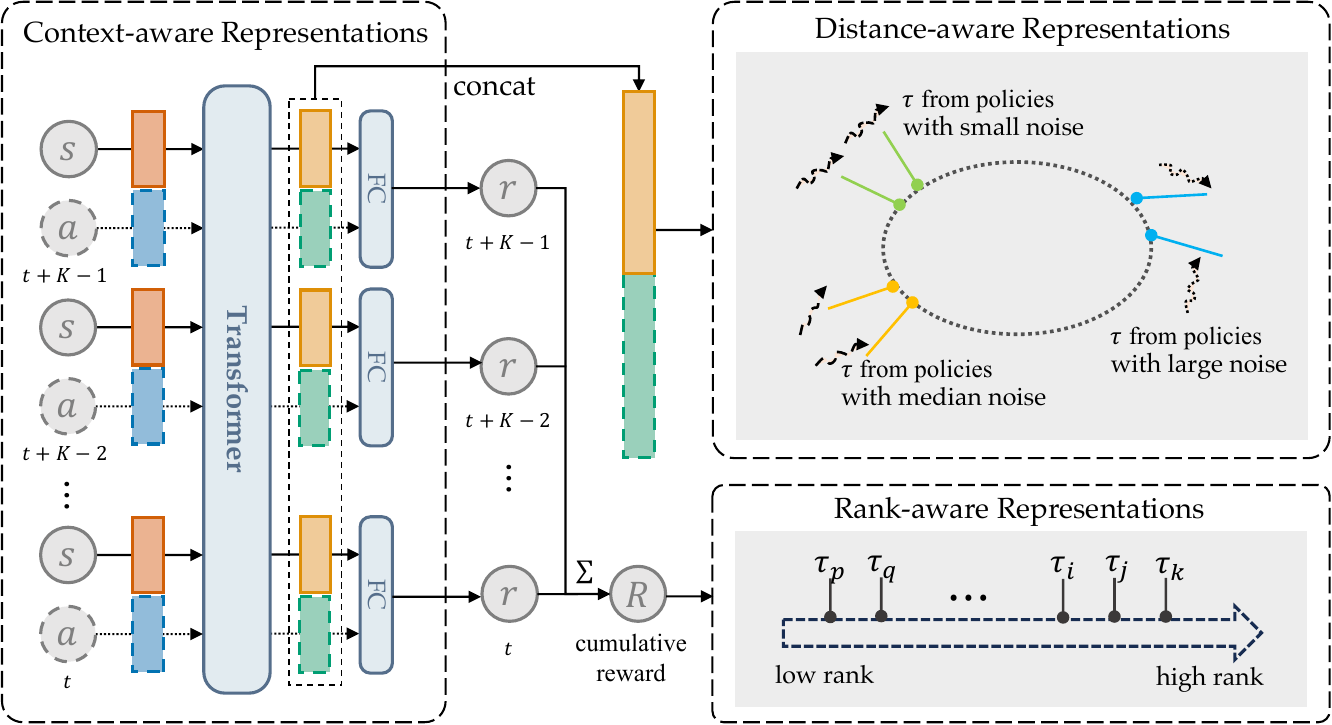}}
\caption{The overall pipeline of DRASRL framework. State-action pairs are initially passed through modality-specific linear embeddings. These embeddings are then fed into a transformer architecture without causal self-attention masks. The transformer generates latent features, which are subsequently processed by two separate branches in the model. The first branch focuses on learning the representations from the degree of difference between trajectory pairs. The second branch is dedicated to learning the latent features from the ranking information of each trajectory.}
\label{fig:drasrl}
\end{center}
\end{figure*}

\section{Related Work}
\textbf{Learning from Demonstration.} Learning from demonstration (LfD) has grown increasingly popular in recent years.
A direct way of learning from demonstration is behavior cloning~\cite{bain1995framework, pan2022backward} in which policies emerge from the demonstrations via supervised learning. However, the performance of the behavior cloning policy is typically bounded by the demonstrator.
Instead of learning a mapping from state to action like behavior cloning, inverse reinforcement learning seeks to find an underlying reward function that explains the expert’s intention and goal, while there is a large body of research. 
Maximum Entropy IRL~\cite{ziebart2008maximum} and Max Margin IRL~\cite{abbeel2004apprenticeship} were proposed to solve the ill-posed problem in IRL.
Furthermore, guided cost learning~\cite{finn2016guided} and adversarial IRL~\cite{fu2017learning} learned the reward model and policy using the generative adversarial framework~\cite{goodfellow2014generative}.

Most of the previous LfD works assume that their demonstrations are provided by experts. However, it is difficult to obtain optimal demonstrations in many real-world tasks. By assuming that only sub-optimal demonstrations are available, our work aims to learn from sub-optimal demonstrations and achieve better-than-demonstrator performance.

\textbf{Learning from Sub-optimal Demonstration.}
Not many works tried to learn good policies from sub-optimal demonstrations. ~\citet{syed2007game} proved that an apprenticeship policy that guarantees superiority over the demonstrator can be found, by knowing which features have a positive or negative effect on the true reward. However, their method requires hand-crafted, linear features and knowledge of true signs.
Recently, T-REX~\cite{brown2019extrapolating} leveraged rank information provided by experts to learn a reward model and significantly outperform the demonstrator. Furthermore, based on T-REX, D-REX~\cite{brown2019drex} automatically generates ranked demonstrations by different noise injections to a behavior cloning policy. 
\citet{chen2021learning} empirically studied the noise-performance relationship and proposed an assumption that the relationship can be described by a four-parameter sigmoid function. Building on this assumption, their Self-Supervised Reward Regression (SSRR) provides state-of-the-art performance on continuous control tasks. ~\citet{cui2021airl} performed an in-depth study on SSRR, and their experimental results showed that the main reason for extrapolating beyond sub-optimal is enforcing the reward function to extrapolate in the direction of ``noise is worse", not the specific form of the reward function.

Our work follows the motivation of D-REX and can be seen as a form of preference-based inverse reinforcement learning (PBIRL) ~\cite{sugiyama2012preference, wirth2017survey} which is based on preference rankings over demonstrations.

\textbf{Self-Supervised Learning.}
Self-supervised learning (SSL) aims to learn rich representations from the enormous available unlabeled data to boost the downstream tasks. These years, SSL achieved many amazing progresses in computer vision ~\cite{he2022masked, grill2020bootstrap, he2020momentum} and natural language processing ~\cite{devlin2018bert, radford2018improving}. Some of them even outperformed the performance with supervised learning ~\cite{he2022masked}. Meanwhile, SSL also widely appears in RL. Specifically, CURL ~\cite{laskin2020curl} builds a contrastive task between different views of the same observations to accelerate the image encoder convergence. SODA ~\cite{hansen2021generalization} maximized the mutual information between augmented and non-augmented data in latent space, which is conducive for generalization in RL. Inspired by the success of SSL, Our method leverages contrastive learning techniques for IRL, aiming at learning a reward model with better generalization.

\section{Preliminaries}
Markov Decision Process (MDP)\cite{sutton2018reinforcement} is a mathematical framework used to model the environment in reinforcement learning (RL). It is typically represented as a tuple $(\mathcal{S}, \mathcal{A}, P, R, \gamma)$. At each time step, the agent observes a state $s_t \in \mathcal{S}$ and subsequently takes action $a_t \in \mathcal{A}$. The next state $s_{t+1}$ is then sampled from a stochastic transition function $P: \mathcal{S} \times \mathcal{A} \times \mathcal{S} \rightarrow [0,1]$, and the agent receives a reward $r_t = R(s_t, a_t)$. $\gamma$ denotes the discount factor. The objective of the agent is to learn a policy $\pi(a_t | s_t)$ that maximizes the expected cumulative reward over time.

In the context of inverse reinforcement learning (IRL), the ground truth reward function is unavailable to train the optimal policy. Furthermore, accessing the optimal expert demonstrations may not be feasible. Instead, the goal is to learn a reliable reward function parameterized by $\theta$, denoted as $R_{\theta}(s_t, a_t)$, using a set of collected sub-optimal demonstration trajectories $D = \{ \tau_1, \ldots, \tau_N \}$, where each trajectory $\tau = (s_0, a_0, s_1, a_1, \ldots)$ represents a sequence of states and actions. Such a reward function needs to generalize beyond the available logged data, allowing for extrapolation. By utilizing $R_{\theta}$, it becomes possible to derive a policy that surpasses the performance of the original demonstrator.

\section{Method}
In this section, we present the Distance-rank Aware Sequential Reward Learning (DRASRL) framework, which aims to incorporate context-awareness and consider both the distances and ranks between trajectories during reward learning. The DRASRL framework enables the creation of structured representations for sub-optimal trajectories, leading to more reliable reward signals. This, in turn, enhances the learning of policies by providing more informative and accurate guidance. 

Our method adopts the D-REX~\cite{brown2019drex} paradigm. D-REX involves utilizing behavior cloning to learn a policy $\pi_{\mathrm{BC}}$ from sub-optimal demonstrations $D$. Following this, various levels of noise are introduced to the outputs of $\pi_{\mathrm{BC}}$, generating trajectories with different performance levels. The underlying assumption is that as the level of noise increases, the performance of the cloned policy deteriorates, eventually converging to a random policy. Consequently, the ranks of the generated trajectories can be estimated based on the injected noise levels. Trajectories with lower levels of noise are assigned higher ranks, indicating better performance. The reward function $R_{\theta}$ is learned using supervised learning, employing a pairwise ranking loss with the ranked trajectories.

However, it is important to note that while D-REX primarily targets the ambiguity in trajectory ranking during reward learning, it does not explicitly address the ambiguity in the degree of difference between trajectories. We contend that both trajectory ranking and the degree of difference in terms of returns between trajectories are crucial factors in effectively distinguishing between each pair of trajectories during reward learning. This comprehensive approach enables us to gain insights into the relative preferences and facilitates the capture of nuanced distinctions between trajectories. Additionally, each singleton transition is strongly influenced by the contextual information present within that same trajectory. Thus, DRASRL incorporates the information of distances and ranks of trajectories into the representation space during sequential reward learning. This means that the representation of trajectories in the latent space is:

i) \textbf{Context-aware:} The transition representation can automatically adapt to the context information;
ii) \textbf{Distance-aware:} The trajectories with smaller dissimilarity are expected to exhibit similar features;
iii) \textbf{Rank-aware:} The feature representation of each trajectory contains sufficient information to determine its rank.

\subsection{Context-aware Representations}
The reward function $R_{\theta}$ is composed of a trainable representation network $\phi(\cdot)$ parameterized by $\theta_{\phi}$ and a linear mapping with parameter $\omega$. Unlike previous approaches that consider singleton reward for each transition, our method leverages sequential modeling to capture the dependencies and context information within trajectories, enhancing the accuracy and effectiveness of the reward function. Specifically, to acquire context information, the representation network takes all state-action pairs of a sub-trajectory of length $K$ as input, and outputs a sequence of context-aware representations, denoted as:
\begin{equation}
    (\bm{x}_{t+1}, \bm{x}_{t+2}, \cdots, \bm{x}_{t+K}) = \phi(\{ (s_{t+i}, a_{t+i}) \}_{i=1}^{i=K}; \theta_{\phi}),
\end{equation}
where $\bm{x}_{t+i}$ ($i\in \{1, \cdots, K \}$) is a $2d_{x}$-dimensional vector that is the concatenation of two $d_{x}$-dimensional vectors corresponding to $s_{t}$ and $a_{t}$, respectively. Each final reward is obtained by mapping the corresponding representation to a scalar using the linear mapping, denoted as $r_{t+i} = \omega^{\mathrm{T}} \bm{x}_{t+i}$.

To incorporate context information into the features, two suitable options are Bidirectional-LSTM~\cite{graves2005framewise} and transformer~\cite{vaswani2017attention}. Both models can take the sequence of state-action pairs as input and generate features that incorporate information from both past and future. However, considering the computational burden and the difficulty of capturing long-term dependencies with LSTM, we opt for the transformer. The transformer model has the advantage of parallel processing for sequence data and the potential to establish long-term dependencies between state-action pairs. Concretely, the raw state $s_{i}$ and action $a_{i}$ are represented by descriptors in different formats ($t$ is omitted for brevity.). To handle these heterogeneous descriptors, we utilize two separate Multi-Layer Perceptron (MLP) blocks. These MLP blocks project the descriptors into homogeneous feature representations: $\hat{\bm{x}}_{j}$ ($j\in[1, 2K]$), where $\hat{\bm{x}}_{2i-1}=f(s_{i}; \theta_{s}) \in \mathbb{R}^{d_{x}}$, $\hat{\bm{x}}_{2i}=f(a_{i}; \theta_{a}) \in \mathbb{R}^{d_{x}}$, $d_{x}$ is the dimension of each vector, and $f(;\theta)$ is an MLP block with its parameters $\theta$. Then, the Scaled Dot-Product Attention is applied to each feature for calculating the relation weight of one feature to another. These relation weights are normalized and used to compute the linear combination of features. Thus, the resulting feature of one state or action embedded by this attention layer is represented as:
\begin{equation}
    \bm{x}_{j}^{\prime} = W^{O} \sum_{m=1}^{2K} \mathrm{softmax} (\frac{(W^{Q}\hat{\bm{x}}_{j})^{\mathrm{T}} W^{K}\hat{\bm{x}}_{m}}{\sqrt{d_{k}}})W^{V}\hat{\bm{x}}_{m},
\end{equation}
where $W^{Q} \in \mathbb{R}^{d_{k} \times d_{x}}$, $W^{K} \in \mathbb{R}^{d_{k} \times d_{x}}$, $W^{V} \in \mathbb{R}^{d_{v} \times d_{x}}$, and $W^{O} \in \mathbb{R}^{d_{x} \times d_{v}}$ are projection matrices, $d_{k}$ and $d_{v}$ are dimensions of projected features. The attention layer and the feed-forward MLP block $f(;\theta_{FF})$ are cascaded together to get the refined features. In addition, the skip-connection operation is applied to both the attention layer and feed-forward MLP to alleviate the gradient vanishment, written as $\bm{x}_{j}^{\prime\prime} = \hat{\bm{x}}_{j} + \bm{x}_{j}^{\prime} + f(\hat{\bm{x}}_{j} + \bm{x}_{j}^{\prime}; \theta_{FF})$. The feature $\bm{x}_i$ for state-action pair $(s_i, a_i)$ is concatenated by the corresponding state feature ${\bm{x}}^{\prime\prime}_{2i-1}$ and action feature ${\bm{x}}^{\prime\prime}_{2i}$. As a result, the cumulative reward (return) $\mathcal{R}$ of this sub-trajectory is represented as: $\mathcal{R} = \sum_{i=1}^{K} r_{i} = \sum_{i=1}^{K} \omega^{\mathrm{T}}\bm{x}_{i}$. In addition, our sequence model is versatile and can also work with state sequences as input. In this case, the networks responsible for action representation can be removed.

\subsection{Distance-aware Representations}
In the field of IRL, it is a challenging task to learn a policy that outperforms the demonstrator's policy by optimizing a reward function inferred from sub-optimal demonstration trajectories. This is because the reward function learned in this manner often lacks the ability to generalize beyond the limitations of the imperfect demonstrators. As a result, it fails to provide a fair evaluation of the agent's immediate behavior. Recent empirical studies have demonstrated that utilizing ranked demonstrations in IRL can effectively reduce the ambiguity in the reward function. It promotes the generalization capabilities of the reward function, and consequently leads to the extrapolation beyond the demonstrators. 

To overcome the difficulty of obtaining the preferences over demonstrations, D-REX~\cite{brown2019drex} introduces a solution by injecting noise into a policy learned through behavioral cloning. This is done using imperfect demonstrations, which allows for the automatic generation of ranked demonstrations. The reward function is learned using a pairwise ranking loss, which leverages trajectories with diverse preferences. However, this paradigm primarily focuses on utilizing the relative ranking of trajectories and may overlook the significance of considering the degree of difference in terms of return between trajectories. We firmly believe that both trajectory ranks and the degree of differences play crucial roles in effectively alleviating ambiguity in reward learning. \emph{The ranking approach allows us to determine which trajectory, $\tau_i$ or $\tau_j$, is preferred, while considering the degree of difference provides insights into the magnitude of preference that $\tau_i$ has over $\tau_j$.} 

The challenge lies in effectively measuring the relative magnitude of preference of one trajectory over another. Given that these trajectories are generated from policies affected by various magnitudes of noise, policies with similar performance are likely to produce trajectories with similar magnitudes of preferences. Conversely, significant differences in executed policies result in substantial variations in the magnitudes of preferences between trajectories. 
We thus make a reasonable assumption that the difference between two trajectories, denoted as $\tau_i$ and $\tau_j$, in terms of returns can be measured by the distance between the corresponding policies from which they are generated. In other words, the difference in returns is proportional to the distance between the policies, denoted as:
\begin{equation}
    |\mathcal{R}_{i} - \mathcal{R}_{j}| \propto \mathrm{Dist}(\pi_{i}||\pi_{j}),
\end{equation}
where $\mathrm{Dist}(\cdot||\cdot)$ is the function used to measure the policy distance. To further support this assumption, we can demonstrate that the difference in terms of returns is upper bounded by the distance between the corresponding policies:

\begin{theorem}
\label{theo:4.1}
    Given two policies $\pi_{i}$ and $\pi_{j}$, the absolute value of the difference of the expectations of discounted returns $J(\pi_{i})$ and $J(\pi_{j})$ can be upper bounded by a linear mapping of the total variation distance between policies:
    \begin{equation}
    \label{equ:theo4.1}
        |J(\pi_{i}) - J(\pi_{j})| \leq \alpha \max_{s} \mathrm{D_{TV}}(\pi_{i}(\cdot|s)||\pi_{j}(\cdot|s)),
    \end{equation}
    where $\alpha = 2|r|_{\mathrm{max}} / (1-\gamma)^{2}$, and $|r|_{\mathrm{max}}$ is the maximum value of the absolute value of reward in the MDP.
\end{theorem}
\begin{proof}
    Please refer to the Appendix for the proof.
\end{proof}
However, it is impractical to directly adopt the maximum of total variation distance over the state space as the measure of rank difference. To address this challenge, we propose a heuristic approximation by using the expectation of the distance over the states from two trajectories, written as:
\begin{equation}
    |\mathcal{R}_{i} - \mathcal{R}_{j}| \propto \mathrm{E}_{s \sim \pi_{i}, \pi_{j}}[ \mathrm{D_{TV}}(\pi_{i}(\cdot|s)||\pi_{j}(\cdot|s))].
\end{equation}
Intuitively, a straightforward approach to learn the distance-aware reward is to treat the policy distance as a regression target for the reward difference. In this approach, the Mean Squared Error (MSE) loss can be utilized and written as:
\begin{equation}
    \mathcal{L}_{\mathrm{mse}} = (|\mathcal{R}_{i} - \mathcal{R}_{j}|-\beta\mathrm{E}_{s \sim \pi_{i}, \pi_{j}}[ \mathrm{D_{TV}}(\pi_{i}(\cdot|s)||\pi_{j}(\cdot|s)) ])^{2},
\end{equation}
where $\beta$ represents a scaling factor. But in practice, it is challenging to experimentally and theoretically determine the coefficient $\beta$ that would lead to efficient reward learning. By comparison, in the representation space, trajectories derived from policies with similar performance should be clustered together, while those originating from policies with a large performance gap should be pushed apart. Inspired by the contrastive learning, we aim to learn distance-aware representations for reward learning. In traditional contrastive learning, the feature of a data point is pushed away from all negative data points and brought closer to positive data points. However, this approach overlooks the fact that the distance between data sources also reflects the differences in representations for each data point. In the context of DRASRL framework, data source is the corresponding policy from which trajectories are generated. Thus, we propose utilizing the total variation distance as a "soft label" for the distance-aware contrastive learning loss, represented as:
\begin{equation}
\label{equ:d_aware}
\begin{split}
    &\mathcal{L}_{d} = - \sum_{n=0}^{N} \mathrm{Dist}(\bm{X}||\bm{X}_{n})\mathrm{log} \frac{\mathrm{exp}(\bm{X}^{\mathrm{T}} \cdot \bm{X}_{n} / \rho) }{\sum_{l=0}^{N} \mathrm{exp}(\bm{X}^{\mathrm{T}} \cdot \bm{X}_{l} / \rho)} \\
    &\mathrm{Dist}(\bm{X}||\bm{X}_{n}) = 1 - \mathrm{Dist}(\pi(\cdot|s)||\pi_{n}(\cdot|s)),
\end{split}
\end{equation}
where $\bm{X} = [\bm{x}_{1}||\cdots||\bm{x}_{K}]$ is the concatenation of the features of state-action pairs within the same trajectory, and $\rho$ is the temperature hyper-parameter. Unlike the InfoNCE loss~\cite{oord2018representation} that utilizes hard binary labels, our distance-aware loss considers $1-\mathrm{Dist}(\pi_{i}||\pi_{j})$ as a soft label. This soft label is used to guide the pulling together and pushing apart of trajectories in the latent space based on the distance between policies affected by different levels of noise. Since the executed policies are noise-injected policies derived from the same behavioral policy $\pi_{\mathrm{BC}}$, the distances $\mathrm{Dist^{cont}}$ and $\mathrm{Dist^{disc}}$ in both continuous and discrete scenarios can be further simplified as:
\begin{equation}
\label{equ:sim_dist}
\begin{split}
    &\mathrm{Dist^{disc}}(\pi_{i}||\pi_{j}) =\mathrm{D_{TV}}(\pi_{i}||\pi_{j}) =|\epsilon_i -\epsilon_j|(1-\frac{1}{|\mathcal{A}|}) \\
    &\mathrm{Dist^{cont}}(\pi_{i}||\pi_{j}) =\mathrm{D_{TV}}(\pi_{i}||\pi_{j}) = |\epsilon_i -\epsilon_j|,
\end{split}
\end{equation}
where $|\mathcal{A}|$ is the cardinality of discrete action space, and $\epsilon \in (0, 1)$, indicating the level of injected noise, is the probability of executing a random action.
\begin{algorithm}[tb]
    \caption{Distance-rank Aware Sequential
Reward Learning}
    \label{alg:algorithm}
    \begin{algorithmic}[1] 
    \REQUIRE demonstrations $D$, noise schedule $\mathcal{E}$, context length $K$\\
    \ENSURE policy $\pi$ \\
        \STATE Obtain policy $\pi_{BC}$ via behavior cloning on demonstrations
        \FOR{$ \epsilon_i \in \mathcal{E}  $}
        \STATE  Collect a set of trajectories ${S}_i:(\tau_0,\tau_1,...)$  from noise-injected policy $\pi_{BC}(\cdot|\epsilon_i)$
        \STATE Initialize full queue $Q_i$  with sub-trajectories of length $K$ sampling from  ${S}_i$.
        \ENDFOR
        
        \FOR{ each iteration }
        \FOR {$ \epsilon_i \in \mathcal{E}  $}
         \STATE sample a  sub-trajectory of length $K$ from $S_i$
         \STATE Optimize distance-rank aware loss function(Eq.\ref{equ:total_loss}) for reward function $R_\theta$ between new sub-trajectory and data in all queues ($Q_0,Q_1,...$).
        \STATE Remove the oldest sub-trajectory from $Q_i$ and add the new sub-trajectory to $Q_i$
        \ENDFOR
        \ENDFOR
        \STATE Optimize policy $\hat{\pi}$ using reinforcement learning with reward function $R_\theta$.
        \STATE \textbf{return} $\hat{\pi}$ 
    \end{algorithmic}
\end{algorithm}

\subsection{Rank-aware Representations}
Although minimizing $\mathcal{L}_{d}$ allows the representation to encode the relative distance between noise-injected policies, it does not explicitly ensure that the representations incorporate information about the ranking of the corresponding trajectories. To learn rank-aware latent features, we can leverage the pairwise ranking loss~\cite{brown2019drex,luce2012individual}, denoted as:
\begin{equation}
\label{equ:r_aware}
   \mathcal{L}_{r} = -\frac{1}{|\mathcal{P}|} \sum_{(i, j) \in \mathcal{P}} \log \frac{\exp \sum\limits_{(s,a)\in\tau_{j}} \omega^{\mathrm{T}} \phi(\tau_{j})}
   {\exp \sum\limits_{(s,a)\in\tau_{i}} \omega^{\mathrm{T}} \phi(\tau_{i}) + \exp \sum\limits_{(s,a)\in\tau_{j}} \omega^{\mathrm{T}} \phi(\tau_{j})},
\end{equation}
where $\mathcal{P}=\left\{(i, j): \tau_i \prec \tau_j \right\}$, $|\mathcal{P}|$ is the cardinality of the set $\mathcal{P}$, and $\tau_i \prec \tau_j$ denotes $\tau_j$ is ranked higher than $\tau_i$.

\begin{table*}[t]
\caption{Comparison of the performance of our method with SSRR and D-REX in MuJoCo tasks. The results are the average
ground-truth returns over 3 seeds. BC refers to the behavior cloning policy used in D-REX and our method. Our DRASRL outperforms other methods by a wide margin.}
\footnotesize
\label{t_mujoco}
\begin{center}
\begin{tabular}{lccccccccc}
 \toprule[1.2pt]
  & \multicolumn{3}{c}{Demonstration} &DRASRL(Ours)&SSRR& D-REX&BC \\\midrule
 Tasks & \#demo& Average  & best & Average & Average& Average& Average\\    
        \midrule       
        HalfCheetah-v3     &1&1129{\small $\pm$0}&1129 &\textbf{4956{\small $\pm$700}}&1991{\small $\pm$734}&1724{\small $\pm$401}&579{\small $\pm$138}\\
        Hopper-v3   &4&1305{\small $\pm$14}&1319&\textbf{2073{\small $\pm$354}}&1477{\small $\pm$584}&1316{\small $\pm$82}&1242{\small $\pm$56}\\
        \bottomrule[1.2pt]
\end{tabular}
\end{center}
\end{table*}

\begin{figure*}[t]
\centering
    \begin{minipage}{0.30\textwidth}
    \includegraphics[width=\linewidth]{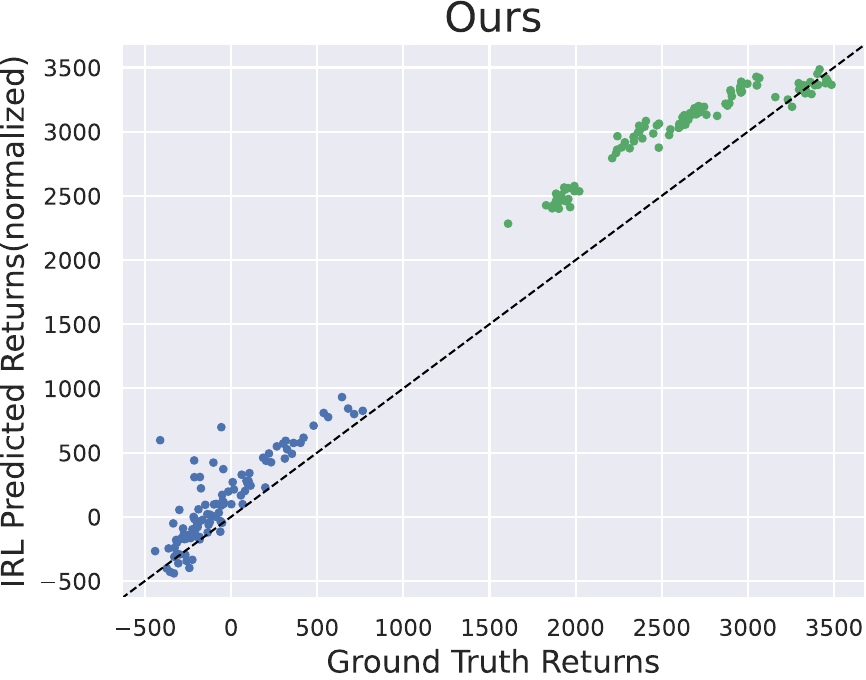}
    \end{minipage}
    \quad
    \begin{minipage}{0.30\textwidth}
    \includegraphics[width=\linewidth]{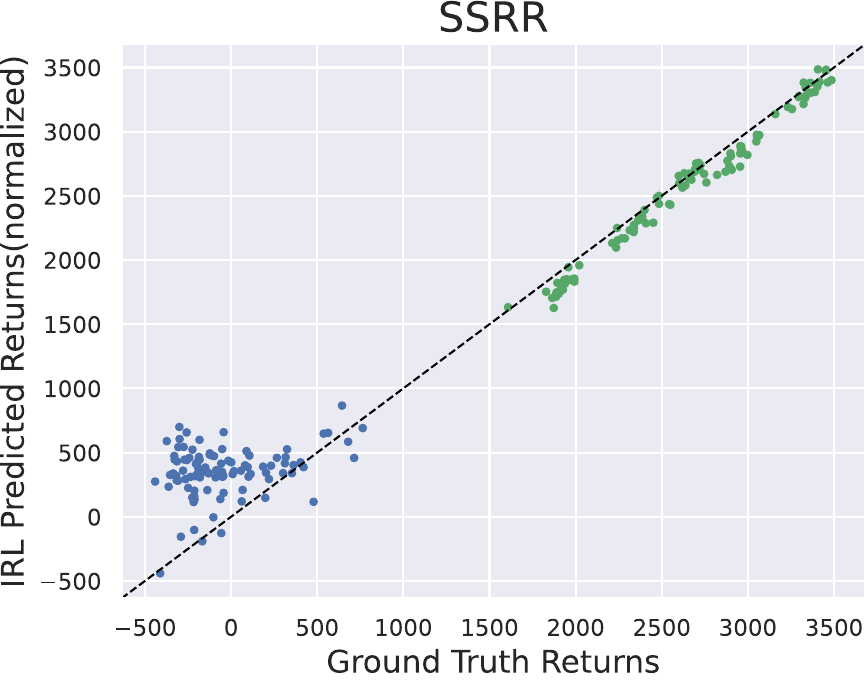}
    \end{minipage}
    \quad
    \begin{minipage}{0.30\textwidth}
    \includegraphics[width=\linewidth]{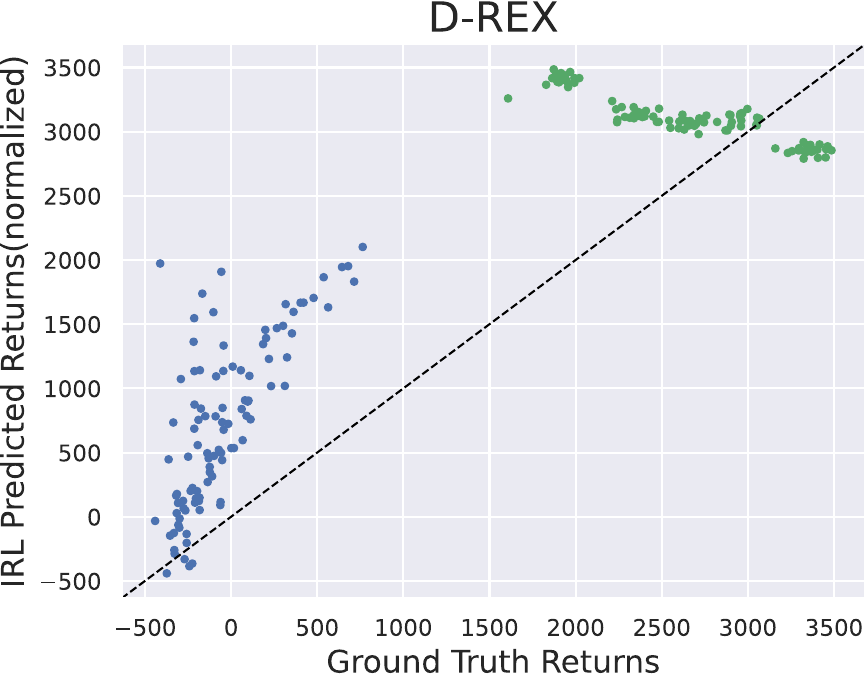}
    \end{minipage}
    \caption{Reward function correlation with ground-truth reward in HalfCheetah-v3. Blue dots represent reward training data generated via behavior cloning policy with different noise injections. Green dots represent the demonstrations not seen during training. Both predicted returns are normalized to the same range as ground-truth returns.
}
    \label{fig: corr_halfcheetah}

\end{figure*}
\subsection{Overall Pipeline}
Our method incorporates both the distance-aware loss and the rank-aware loss. By considering both aspects, we can further reduce the ambiguity in reward learning, leading to improved downstream optimization of the policy. The overall learning objective for the reward function can be written as:
\begin{equation}
\label{equ:total_loss}
    \mathcal{L}(\theta_{\phi}, w) = \mathcal{L}_{d}(\theta_{\phi}) + \lambda \mathcal{L}_{r}(\theta_{\phi}, w),
\end{equation}
where $\lambda$ is a hyper-parameter to balance two loss terms.

To train the reward model using the objective in Equation~\ref{equ:total_loss}, we employ the dictionary look-up technique proposed in~\cite{he2020momentum}. Initially, behavior cloning policies injected with different levels of noise are stored to automatically generate trajectories. Each policy corresponds to a queue, which will be used to store the sub-trajectories generated by that policy. In each training iteration, we randomly sample a sub-trajectory from each queue, indicating a specific level of noise. These sub-trajectories are used to compute the distance-rank aware loss, as described in Equation~\ref{equ:d_aware}. And the sampled sub-trajectories are ordered based on their ranks to compute the rank-aware loss in Equation~\ref{equ:r_aware}. After computing the loss, the sampled sub-trajectories are removed from their respective queues. New sub-trajectories, generated by executing policies with the same noise levels, are then added to the corresponding queues. The overall pipeline is depicted in Figure~\ref{fig:drasrl}. The pseudocode of DRASRL framework is illustrated in Algorithm~\ref{alg:algorithm}.

\section{Experiments}

Our experiments are aimed to study the following problem:
\begin{itemize}
\item \textbf{Performance}: Can our method learn a more accurate reward model compared to existing approaches and improve the performance of RL policies? (see Table~\ref{t_mujoco}, \ref{t_coeff}, \ref{t_atari})
\end{itemize}
\begin{itemize}
 \item
\textbf{Influence of transformer architecture}: Does the transformer architecture contribute to the learning of a more effective reward model? (see Table~\ref{t_ablation_cl})
\end{itemize}
\begin{itemize}
 \item
\textbf{Influence of context length}: Intuitively, context length is an important hyper-parameter for our DRASRL. We set different context lengths and compare their performance. (see Table~\ref{t_length})
\end{itemize}
\begin{itemize}
 \item
\textbf{Influence of contrastive learning}: Does contrastive learning help our method to learn better representation? (see Table~\ref{t_ablation_cl} and Figure~\ref{fig:tsne})
\end{itemize}

We evaluate our method on a range of tasks including MuJoCo robot locomotion tasks and Atari video games. We take the previous learning from sub-optimal demonstration methods (D-REX, SSRR) as our baseline for comparison.
\begin{table*}[t]
\caption{Comparison of the performance of our method with D-REX and the demonstrator’s performance. The results are the average ground-truth returns over 3 seeds. Bold denotes the best performance between IRL methods.}
\label{t_atari}
\begin{center}
\begin{tabular}{lccccc}
\toprule[1.2pt]
  & \multicolumn{2}{c}{Demonstrations} &DRASRL(Ours) & D-REX & BC \\\midrule
 Tasks & Average & Best & Average & Average&Average \\
\midrule
        Beam Rider     & 744{\small $\pm$173} &1092        & \textbf{6157{\small $\pm$792}}     & 4675{\small $\pm$1678}& 502{\small $\pm$205}  \\
        Breakout  &   68{\small $\pm$57} &230       & \textbf{262{\small $\pm$149}}  & 234{\small $\pm$157}&  5{\small $\pm$6} \\
          Pong        & 15.3{\small $\pm$3.6} &20         & \textbf{17.8{\small $\pm$5.9}}  &    1.3{\small $\pm$14.0} & 7.9{\small $\pm$9.5} \\
          Q*bert        & 788{\small $\pm$144} &875         & \textbf{125903{\small $\pm$47856}}  &    21934{\small $\pm$8230} & 728{\small $\pm$173}\\
          Seaquest        & 550{\small $\pm$149} &780        & \textbf{995{\small $\pm$97}} &    773{\small $\pm$26} &408{\small $\pm$125}   \\
          Space Invaders        & 478.5{\small $\pm$267} &880         & \textbf{1641{\small $\pm$664}} &    806{\small $\pm$189} &322{\small $\pm$161}  \\
\bottomrule[1.2pt]
\end{tabular}
\end{center}
\end{table*}
\subsection{MuJoCo Tasks}

\textbf{Experimental setup.} 
 We evaluate our method on two MuJoCo simulated robot locomotion tasks:  Hopper-v3 and HalfCheetah-v3. 
The agent is expected to maintain balance and move forward.
To generate sub-optimal demonstrations, Proximal Policy Optimization(PPO)~\cite{schulman2017proximal} agents were partially trained on the two tasks with ground truth reward.
For our method, state sequence $(s_1,s_2,...,s_K)$ is fed to the reward model. The length of the state sequence is 5, $\rho$ in distance-aware loss (Eq.\ref{equ:d_aware}) is 1.0 and $\lambda$ in total loss function (Eq.\ref{equ:total_loss}) is 0.1. We use 20 different noise levels equal-spaced between $[0,1)$, and generate 5 trajectories for each noise level.
When computing rank-aware loss term, we discard pairs whose noise difference is smaller than 0.3 to stabilize the training process. 
The reward model is trained by an Adam optimizer with a learning rate of 0.0001, weight decay of 0.01, and batch size of 64 for 150 iterations. 

When RL policies are trained with a reward model, the predicted rewards are normalized by a running mean and standard deviation. Additionally, a control penalty is added to the normalized reward. This control penalty represents a safety prior over reward functions which is used in OpenAI Gym\cite{brockman2016openai}. We set the same control coefficient as the default value in Gym. 

\textbf{Learned policy performance.} We trained a PPO agent with reward models for 10M environment steps in 3 different seeds and report their average performance (ground truth returns) in Table \ref{t_mujoco}. It compares the scores of RL policies that were trained with reward models. The results demonstrate that our DRASRL outperforms other methods when combined with the same RL algorithm by a wide margin.

\begin{table}
  \caption{Learned Reward Correlation Coefficients with Ground-Truth Reward. Our DRASRL learned a more accurate reward model with higher correlation coefficients.}
\label{t_coeff}
\begin{center}
\begin{tabular}{lcccccc}
\toprule[1.2pt]
  &  \multicolumn{1}{c}{DRASRL(Ours)} & \multicolumn{1}{c}{D-REX} &\multicolumn{1}{c}{SSRR}\\\midrule
        HalfCheetah-v3  & \textbf{0.989}    & 0.900& 0.986 \\
        Hopper-v3 & \textbf{0.884} &0.412& 0.822 \\
\bottomrule[1.2pt]
\end{tabular}
\end{center}
\end{table}

\textbf{Correlation with ground-truth reward.} To investigate how well our method recovers the reward function, we compare predicted returns and ground truth returns of a number of trajectories. Following~\cite{chen2021learning}, we use the reward training dataset and 100 episodes of unseen trajectories that most of which have better performance than the best trajectories of training data to compute correlation coefficients. We present the correlation coefficients between IRL predicted returns and ground-truth returns for our method and baseline methods in Table \ref{t_coeff} and Figure~\ref{fig: corr_halfcheetah}.
The findings demonstrate that our proposed method not only accurately recovers the reward function within the scope of the training data but also effectively extrapolates beyond it. 

\subsection{Atari Tasks}


\textbf{Experimental setup.} Atari environments have high-dimensional visual input, normally $4\times84\times84$ dimension, which is much higher than MuJoCo continuous control tasks. Such high-dimensional states are a challenge for IRL methods. Many previous methods, (\textit{e.g.} AIRL), are not able to get good performance in Atari~\cite{tucker2018inverse}. For SSRR, since the authors did not report scores on atari in the paper and we cannot reproduce meaningful results (probably because it is based on AIRL), we only compare our method with D-REX. We used a convolutional encoder to extract the visual features from observation before feeding them to the transformer.
We used specific values for the hyper-parameters $\lambda$ and $\rho$ for different Atari games. For Beam Rider and Space Invaders, the values of $\lambda$ and $\rho$ were set to 0.1 and 0.1, respectively. For Breakout, Q*bert, and Pong, the values of $\lambda$ and $\rho$ were set to 0.1 and 1.0, respectively. For Seaquest, the values of $\lambda$ and $\rho$ were set to 0.01 and 1.0, respectively.

To generate sub-optimal demonstrations, a policy is trained insufficiently using the PPO algorithm with the ground truth reward, resulting in the generation of 10 trajectories. Then, the behavior cloning policy learns from these demonstrations. Furthermore, the noise schedule $\mathcal{E} =[1.0,0.75,0.5,0.25,0.02]$ is used to generate 20 trajectories for each level of noise. For all Atari environments, state-action sequence $(s_1,a_1,s_2,a_2,...,s_K,a_K)$ is fed to the reward model, and the context length $K$ is set to 20.

In Atari games, the score can significantly influence reward learning, often leading the reward model to focus primarily on the score. To mitigate the impact of this score information, we mask all score-related pixels during both reward learning and prediction.



\textbf{Learned policy performance.} As shown in Table \ref{t_atari}, we evaluate our method on six Atari games and measure the performance by the RL policy learned with the reward model.
To avoid reward scaling issues, we feed the predicted rewards to a sigmoid function before passing it to the RL algorithm. We train a PPO agent with our reward model for 50M frames in 3 different seeds and report their average performance (ground truth returns) in Table \ref{t_atari}. Our DRASRL outperforms the baseline method in all six Atari environments. 
In the Q*bert environment, both methods find a known loophole in the game which leads to nearly infinite points (rewards), but our method gets higher points. In other environments, policies trained with our reward model outperform D-REX's policies as well.

\subsection{Ablation Study}
\textbf{Influence of transformer architecture.} By comparing DRASRL without contrastive learning to D-REX, we are able to examine the impact of the transformer architecture. As illustrated in Table~\ref{t_ablation_cl}, employing the transformer architecture leads to a significant improvement in the trained policy's performance.

\textbf{Influence of context length.}
Table~\ref{t_length} shows the performance of the policy trained on HalfCheetah-v3 and Space Invaders using our DRASRL under different context lengths. The main trend observed is that as the context length increases, the performance improves. However, in the case of HalfCheetah, we observe a deterioration in performance when the context length is increased from 5 to 10. We attribute this decrease in performance to overfitting.
\begin{table}[t]
    \caption{The performance of policies trained by DRASRL with different context lengths.}
    \centering
    \begin{tabular}{lccc}
    \toprule[1.2pt]
    Tasks & $K=1$  & $K=5$ & $K=10$ \\
    \midrule
    HalfCheetah-v3  & 3923{\small $\pm$1054}  &4956{\small $\pm$700}  & 2884 {\small $\pm$657} \\
    \midrule
    Space Invaders  &734{\small $\pm$270} & 1079{\small $\pm$388} & 1282{\small $\pm$674} \\
    \bottomrule[1.2pt]
    \end{tabular}
    \label{t_length}
\end{table}

\textbf{Influence of contrastive learning.}
We also analyzed the effect of introducing contrastive learning. As shown in Table~\ref{t_ablation_cl}, Distance-aware loss enabled us to achieve performance improvements in several Atari games and MuJoCo locomotion tasks, especially in Q*bert, where the score is increased to approximately three times over the original score. More analysis of feature space can be seen in the next section.

\begin{table}
  \caption{A comparison of performance outcomes, illustrating the effects of employing or not employing contrastive distance-aware loss, and the use or non-use of the transformer architecture. Bold denotes the best performance.}
\label{t_ablation_cl}
\setlength\tabcolsep{0.9pt}
\begin{center}
\begin{tabular}{lcccc}
\toprule[1.2pt]
  &DRASRL & DRASRL(w/o CL) &D-REX\\\midrule
        Beam Rider             & \textbf{6157{\small $\pm$792}}     &   5321{\small $\pm$878} &4678{\small $\pm$1678}\\
        Breakout           & \textbf{262{\small $\pm$149}} & 261{\small $\pm$200} &234{\small $\pm$157}  \\
          Pong                & 17.8{\small $\pm$5.9}  &  \textbf{19.8{\small $\pm$2.0}}  &1.3{\small $\pm$14.0}  \\
          Q*bert                 & \textbf{125903{\small $\pm$47856}}  &  45284{\small $\pm$5401} &21934{\small $\pm$8230}   \\
          Seaquest               & \textbf{995{\small $\pm$97}} &   818{\small $\pm$12}  &773{\small $\pm$26}  \\
          Space Invaders           & \textbf{1641{\small $\pm$664}} &  1388{\small $\pm$392}&806{\small $\pm$189}  \\
          \midrule
          HalfCheetah &\textbf{4960{\small $\pm$700}} &1898{\small $\pm$64}&1724{\small $\pm$401}\\
          Hopper &\textbf{2073{\small $\pm$354}}&1421{\small $\pm$36}&1316{\small $\pm$82}\\
\bottomrule[1.2pt]
\end{tabular}
\end{center}
\end{table}

\begin{figure}[h]
\centering
\includegraphics[width=0.8\textwidth]{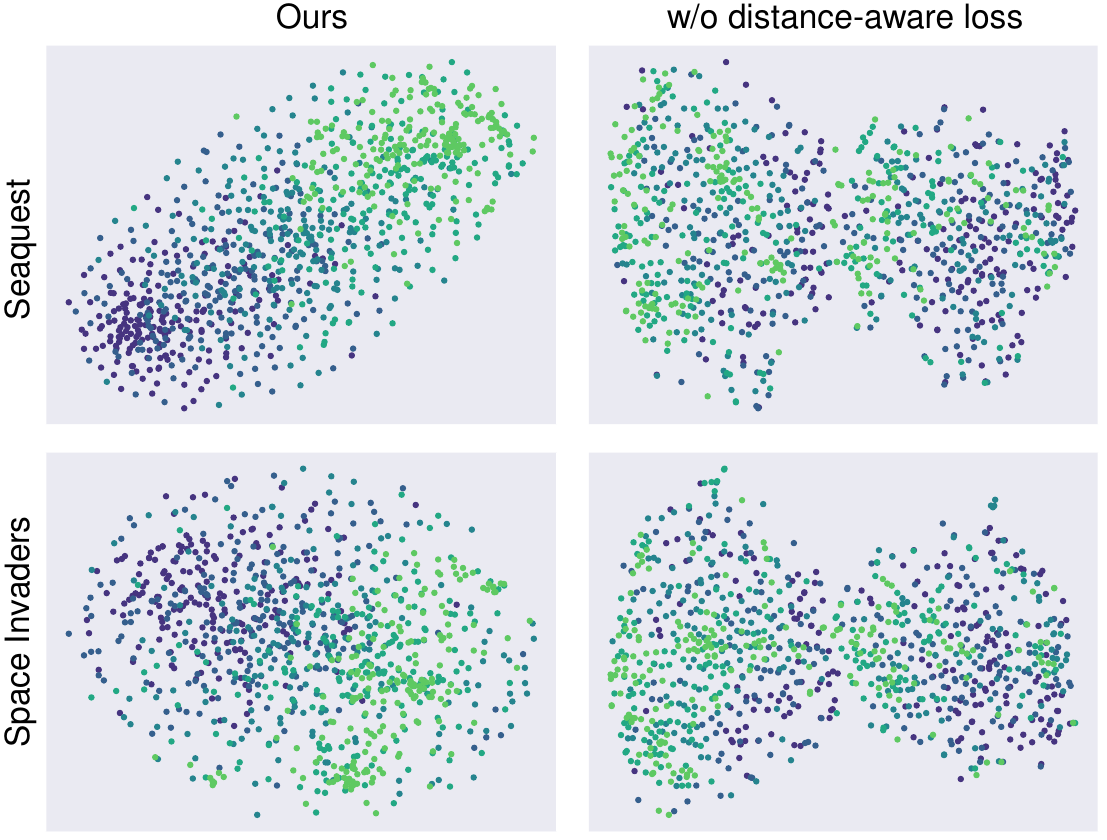}
\caption{Feature visualization using t-SNE with or without distance-aware loss. Each sub-trajectory is visualized as a point. Brighter colors represent less noise injection. Profiting from distance-aware loss, sub-trajectories are orderly arranged according to their performance leading to better generalization.}
\label{fig:tsne}
\end{figure}
\subsection{Visualization Analysis}
To see how contrastive learning influences the structure of representation space,
We randomly select 200 sub-trajectories for each noise level and visualize the corresponding features extracted by both DRASRL and the version of DRASRL without contrastive learning. By leveraging t-SNE\cite{van2008visualizing}, we visualize data points in the representation space as shown in Figure \ref{fig:tsne}. Each point represents a sub-trajectory, and different colors depict different noise levels. A brighter color means less noise injection. We observe that optimizing our distance-aware loss function leads to more separable features. Specifically, trajectories with similar performances are pulled together in feature space, while trajectories with totally different performances are pushed apart. 
On the contrary, trajectories are randomly distributed in the feature space without distance-aware loss.
Structured feature space where trajectories are orderly arranged according to their performance or levels of injected noise leads to better generalization.

\section{Conclusion}

In this paper, we develop a novel inverse reinforcement learning framework that learns reward functions at the sequence level. In addition, we develop a novel distance-rank aware loss leading to a structured feature space with better generalization.
By combining them, our method outperforms previous methods in a range of tasks.
To the best of our knowledge, we are the first to introduce transformer and contrastive learning to the IRL.

\textbf{Limitation and future work.}
Given that we trained separate reward models for different tasks, a limitation of our method is that we might not be fully leveraging the potential of the transformer architecture for transfer between tasks. We believe that a promising future direction could involve training larger models to solve multiple tasks that share overlapping domains.

\bibliographystyle{apalike} 
\bibliography{neurips_2023}

\setcounter{theorem}{0}
\setcounter{proof}{0}
\setcounter{equation}{0}
\clearpage
\appendix
\begin{appendices}
\setcounter{section}{0}
\setcounter{equation}{0}
\renewcommand\thesection{\Alph{section}} 
{\centering \section*{\Large Appendix}}
\section{Proof of theorem 4.1}

\begin{lemma}\label{lemma1} [\cite{pmlr-v70-achiam17a}, Appendix.10.1.1]
Given a policy $\pi(\cdot|s)$, a transition model $p(\cdot|s,a)$, and a initial state distribution $\mu(s)$, the discounted state visitation distribution $d_{\pi}(s)$ under policy $\pi$ can be described as:
\begin{equation}
    d_{\pi}(s) = (1-\gamma)\mu(s) + \gamma\sum_{s^{\prime},a} d_{\pi}(s^{\prime})\pi(a|s^{\prime})p(s|s^{\prime},a)
\end{equation}
\end{lemma}

\begin{theorem}
    Given two policies $\pi_{i}$ and $\pi_{j}$, the absolute value of the difference of their returns $J(\pi_{i})$ and $J(\pi_{j})$ can be upper bounded by a linear mapping of the total variation distance between policies:
    \begin{equation}
    \label{eq20}
        |J(\pi_{i}) - J(\pi_{j})| \leq \alpha \max_{s} \mathrm{D_{TV}}(\pi_{i}(\cdot|s)||\pi_{j}(\cdot|s)),
    \end{equation}
    where $\alpha = 2|r|_{\mathrm{max}} / (1-\gamma)^{2}$.
\end{theorem}

\begin{proof}
    The absolute value of the difference of returns derived from $\pi_{i}$ and $\pi_{j}$ can be represented as:
    \begin{equation}
    \begin{split}
        |J(\pi_{i}) - J(\pi_{j})| 
        &= |\mathrm{E}_{s \sim \mu(s)} [V^{\pi_{i}}(s) - V^{\pi_{j}}(s)] | \\
        &\leq \mathrm{E}_{s \sim \mu(s)} |V^{\pi_{i}}(s) - V^{\pi_{j}}(s)|,
    \end{split}
    \end{equation} 
    where $\mu(s)$ is the initial state distribution.

    Here, we set $\Delta V(s) = V^{\pi_{i}}(s) - V^{\pi_{j}}(s)$, and it can be transformed as:
    \begin{equation}
    \label{eq22}
    \begin{split}
        \Delta V(s) &= \mathrm{E}_{a \sim \pi_{i}} [r(s, a) + \gamma \mathrm{E}_{s^{\prime}} (V^{\pi_{i}}(s^{\prime})) ] \\
        &\quad - \mathrm{E}_{a \sim \pi_{j}} [r(s, a) + \gamma \mathrm{E}_{s^{\prime}} (V^{\pi_{j}}(s^{\prime})) ] \\
        &=\underbrace{\mathrm{E}_{a \sim \pi_{i}} (r(s, a)) - \mathrm{E}_{a \sim \pi_{j}} (r(s, a))}_{\textcircled{1}} \\
        &\quad +\gamma\underbrace{\mathrm{E}_{a \sim \pi_{i}}\mathrm{E}_{s^{\prime}} (V^{\pi_{i}}(s^{\prime})) -\gamma\mathrm{E}_{a \sim \pi_{j}}\mathrm{E}_{s^{\prime}} (V^{\pi_{j}}(s^{\prime}))}_{\textcircled{2}}
    \end{split}
    \end{equation}
    The term $\textcircled{1}$ can be upper bounded by:
    \begin{equation}
    \label{eq23}
    \begin{split}
        &\mathrm{E}_{a \sim \pi_{i}} (r(s, a)) - \mathrm{E}_{a \sim \pi_{j}} (r(s, a)) \\
        = &\sum_{a} [\pi_{i}(a|s) - \pi_{j}(a|s)] r(s,a) \\
        \leq &\sum_{a} |\pi_{i}(a|s) - \pi_{j}(a|s)| |r(s,a)| \\
        \leq &2|r|_{\mathrm{max}} \frac{1}{2}\sum_{a} |\pi_{i}(a|s) - \pi_{j}(a|s)| \\
        = &2|r|_{\mathrm{max}} \mathrm{D_{TV}}(\pi_{i}(\cdot|s) || \pi_{j}(\cdot|s)) = \delta_{1}(s)
    \end{split}
    \end{equation}
    The upper bound of term $\textcircled{2}$ can be firstly written as:
    \begin{equation}
    \label{eq24}
    \begin{split}
        &\mathrm{E}_{a \sim \pi_{i}}\mathrm{E}_{s^{\prime}} (V^{\pi_{i}}(s^{\prime})) - \mathrm{E}_{a \sim \pi_{j}}\mathrm{E}_{s^{\prime}} (V^{\pi_{j}}(s^{\prime})) \\
        =&\mathrm{E}_{a \sim \pi_{i}}\mathrm{E}_{s^{\prime}} (V^{\pi_{i}}(s^{\prime})) - \mathrm{E}_{a \sim \pi_{i}}\mathrm{E}_{s^{\prime}} (V^{\pi_{j}}(s^{\prime})) \\
        &+\mathrm{E}_{a \sim \pi_{i}}\mathrm{E}_{s^{\prime}} (V^{\pi_{j}}(s^{\prime})) - \mathrm{E}_{a \sim \pi_{j}}\mathrm{E}_{s^{\prime}} (V^{\pi_{j}}(s^{\prime})) \\
        =&\mathrm{E}_{a \sim \pi_{i}}\mathrm{E}_{s^{\prime}} \Delta V(s^{\prime}) + \sum_{a}(\pi_{i}(a|s) - \pi_{j}(a|s))\mathrm{E}_{s^{\prime}} V^{\pi_{j}}(s^{\prime})\\
    \end{split}
    \end{equation}
    We then derive an upper bound of the second term in Equation~\ref{eq24} as:
    \begin{equation}
    \label{eq25}
    \begin{split}
        &\sum_{a}(\pi_{i}(a|s) - \pi_{j}(a|s))\mathrm{E}_{s^{\prime}} V^{\pi_{j}}(s^{\prime}) \\
        \leq &\sum_{a}(\pi_{i}(a|s) - \pi_{j}(a|s)) \max_{s^{\prime}} V^{\pi_{j}} (s^{\prime}) \\
        \leq &\frac{2|r|_{\mathrm{max}}}{1-\gamma} \frac{1}{2}\sum_{a} |\pi_{i}(a|s) - \pi_{j}(a|s)| \\
        = &\frac{2|r|_{\mathrm{max}}}{1-\gamma} \mathrm{D_{TV}}(\pi_{i}(\cdot|s)||\pi_{j}(\cdot|s)) = \delta_{2} (s)
    \end{split}
    \end{equation}
    Next, Equation~\ref{eq23}~\ref{eq24} and~\ref{eq25} can be plugged into Equation~\ref{eq22}, written as:
    \begin{equation}
        \Delta(s) \leq \delta_{1}(s) + \gamma\delta_{2}(s) + \gamma\mathrm{E}_{a \sim \pi_{i}}\mathrm{E}_{s^{\prime}} \Delta V(s^{\prime})
    \end{equation}
    Thus, the expectation of $\Delta(s)$ can be further upper bounded by:
    \begin{equation}
    \label{eq27}
    \begin{split}
        &\mathrm{E}_{s \sim d_{\pi_{i}}} \Delta(s) \\
        \leq &\mathrm{E}_{s \sim d_{\pi_{i}}} [\delta_{1}(s) + \gamma\delta_{2}(s) + \gamma\mathrm{E}_{a \sim \pi_{i}}\mathrm{E}_{s^{\prime}} \Delta V(s^{\prime})] \\
        \leq &\mathrm{max}_{s}[\delta_{1}(s) + \gamma\delta_{2}(s)] + \gamma\mathrm{E}_{s \sim d_{\pi_{i}}}\mathrm{E}_{a \sim \pi_{i}} \mathrm{E}_{s^{\prime}} \Delta V(s^{\prime}),
    \end{split}
    \end{equation}
    where $d_{\pi}(s)$ is the discounted state visitation distribution under policy $\pi$. 
    
    According to Lemma~\ref{lemma1}, Inequality~\ref{eq27} can be simplified as:
    \begin{equation}
    \begin{split}
        \mathrm{E}_{s \sim d_{\pi_{i}}} \Delta(s) &\leq \mathrm{max}_{s}[\delta_{1}(s) + \gamma\delta_{2}(s)] \\
        &\quad + \sum_{s^{\prime}}[d_{\pi_{i}}(s^{\prime}) - (1-\gamma)\mu(s^{\prime})] \Delta V(s^{\prime})
    \end{split}
    \end{equation}
    Consequently, $\mathrm{E}_{s \sim d_{\pi_{i}}} \Delta(s)$ can be eliminated, and we can obtain:
    \begin{equation}
    \label{eq29}
    \begin{split}
        &(1-\gamma)\mathrm{E}_{s \sim \mu(s)} \Delta V(s) \\
        \leq &\mathrm{max}_{s} [\delta_{1}(s) + \delta_{2}(s)] \\
        = &\frac{2|r|_{\mathrm{max}}}{1-\gamma} \mathrm{max}_{s} \mathrm{D_{TV}} (\pi_{i}(\cdot|s) || \pi_{j}(\cdot|s))
    \end{split}
    \end{equation}
    By taking Equation~\ref{eq29} into Equation~\ref{eq20}, we get the result:
    \begin{equation}
    \label{eq20}
        |J(\pi_{i}) - J(\pi_{j})| \leq \frac{2|r|_{\mathrm{max}}}{(1-\gamma)^{2}} \max_{s} \mathrm{D_{TV}}(\pi_{i}(\cdot|s)||\pi_{j}(\cdot|s)),
    \end{equation}
\end{proof}
\end{appendices}


\end{document}